\documentclass[twoside,11pt]{article}

\usepackage{jmlr2e}
\usepackage{amssymb}
\usepackage{amsmath}
\usepackage{amsfonts}

\usepackage{amsthm}
\usepackage{bbm}
\usepackage{csquotes}
\usepackage{chngcntr}
\usepackage{apptools}
\usepackage[square,numbers]{natbib}
\setcitestyle{square}
\newtheorem*{mytheorem*}{Theorem}
\newtheorem{mytheorem}{Theorem}
\newtheorem{mydef}{Definition}
\newtheorem{mylemma}{Lemma}
\newtheorem*{mylemma*}{Lemma}

\newtheorem{examp}{Example}[section]
\AtAppendix{\counterwithin{mydef}{section}}
\AtAppendix{\counterwithin{mytheorem}{section}}
\AtAppendix{\counterwithin{mylemma}{section}}

\newcommand{\A}{\mathcal{A}}
\newcommand{\B}{\mathcal{B}}
\newcommand{\X}{\mathcal{X}}
\newcommand{\Y}{\mathcal{Y}}

\newcommand{\Le}{\mathcal{L}}
\newcommand{\E}{\mathbb{E}}
\newcommand{\Prob}{\mathbb{P}}

\newcommand{\ml}[2]{\mathcal{L}\left(#1  \!\!  \to  \!\!   #2\right)} % maximal leakage

\firstpageno{1}

\begin{document}
\addtocounter{page}{-1}

%\setcitestyle{square}
\title{A New Approach to Adaptive Data Analysis and Learning via Maximal Leakage}

\author{\name Amedeo Roberto Esposito \email amedeo.esposito@epfl.ch \\
       \addr School of Computer and Communication Sciences \\
       EPFL
       \AND Michael Gastpar \email michael.gastpar@epfl.ch \\
       \addr School of Computer and Communication Sciences \\
       EPFL
       \AND
       \name Ibrahim Issa \email ibrahim.issa@aub.edu.lb \\
       \addr Electrical and Computer Engineering Department \\
       American University of Beirut}

\maketitle
\thispagestyle{empty}

\begin{abstract}
	There is an increasing concern that most current published research findings are false. The main cause seems to lie in the fundamental disconnection between theory and practice in data analysis. While the former typically relies on statistical independence, the latter is an inherently adaptive process: new hypotheses are formulated based on the outcomes of previous analyses. A recent line of work tries to mitigate these issues by enforcing constraints, such as differential privacy, that compose adaptively while degrading gracefully and thus provide statistical guarantees even in adaptive contexts. 
	Our contribution consists in the introduction of a new approach, based on the concept of Maximal Leakage, an information-theoretic measure of leakage of information.
	The main result allows us to compare the probability of an event happening when adaptivity is considered with respect to the non-adaptive scenario. The bound we derive represents a generalization of the bounds used in non-adaptive scenarios (e.g., McDiarmid's inequality for $c$-sensitive functions, false discovery error control via significance level, etc.), and allows us to replicate or even improve, in certain regimes, the results obtained using Max-Information or Differential Privacy. 
	
	In contrast with the line of work started by Dwork et al., our results do not rely on Differential Privacy but are, in principle, applicable to every algorithm that has a bounded leakage, including the differentially private algorithms and the ones with a short description length.
\end{abstract}


\newpage

\section{Introduction}
\label{sec:intro}

There is an increasing concern that most current published research findings are false~\citep{falseResearch,statistCrisis}. This ``crisis'' is mainly due to the difficulty of analyzing large amounts of data. In particular, the statistical inference theory typically assumes that the tests/procedures to be performed are selected before the data are gathered, i.e., statistical independence. By contrast, the data analysis practice is an inherently adaptive process: new hypotheses and analyses are formulated based on the outcomes of previous tests on the same data. 

To circumvent this issue, one could collect fresh samples of data for every new test performed, but this is often expensive and impractical. Alternatively, one could naively divide the original dataset into smaller subsets, and apply each new test to a new subset. However, this severely limits the amount of data available for each test, which in turn negatively affects its accuracy. A more common approach is to save one subset for testing, called the holdout set~\citep{genAdap}, and to reuse it multiple times. This may, however, lead to overfitting to the holdout set itself. This can be observed, for example, in the machine learning competitions Kaggle~\citep{genAdap} or ImageNet~\citep{statValidity}. In these competitions, the participants are given a sequence of samples and are required to provide a model with good prediction capabilities. The model is then tested on a hidden test set by the organizers, and a score is returned to the participants, who can then submit a new model. This procedure is repeated until the end of the competition, wherein the best models are evaluated on yet a different test set: the contrast between the scores obtained at the end and the ones from the continuously reused test set shows significant evidence of overfitting.

%To Amedeo: need more references like stability
The difference between the performance of the model on the training data versus a fresh sample of data is called the generalization error in the machine learning literature. Standard approaches~\citep{stability} to bound this error rely on the notion of \emph{stability} of learning algorithms. Roughly speaking, an algorithm generalizes well if small changes in its input lead to only small changes in its output (i.e., it is stable). A more recent line of work~\citep{statValidity,genAdap,algoStability}, initiated by Dwork \emph{et al.}~\citep{statValidity}, takes adaptivity explicitly into account. These works mainly rely on the notion of \emph{differential privacy} (which originated in the database security literature~\citep{DworkCalibrating}). Besides inducing a suitable notion of stability, differentially private algorithms behave well under composition. That is, a sequence of differentially private algorithms induces a differentially private algorithm, allowing us to make generalization guarantees even in the adaptive setting. However, as it comes from the privacy literature, differential privacy is quite restrictive, which led to the introduction of more relaxed notions (that still behave well under composition) such as $\beta$-max information, $I_\infty^{\beta}(X;Y)$ ~\citep{genAdap,maxInfo}. 

%A great effort has been spent, recently \citep{statValidity,genAdap,maxInfo}, in order to address or, at least, mitigate this problem. Malpractice in the context of data analysis seems to be the prime culprit. The  Often, mutually dependent studies use the same datasets, rendering unclear the statistical validity of such findings and the extent to which the results are \emph{over-fitted} to the data they used.

%The typical approach to use in order to avoid these problems would be the collection of fresh samples from the same data distribution, whenever an employment of a data-dependent procedure is needed \citep{statValidity}. The continuous gathering of new samples is often expensive or even impractical, resulting then in the random splitting of the available data in multiple disjoint sets, ahead of the analysis phase. The case in which numerous, adaptively chosen, procedures are used, this would imply a significant reduction of the amount of data available for each procedure, rendering the generalization guarantees weak, as strictly dependent on the number of samples used.

% What actually happens is that a slice of data is saved for the purpose of testing (holdout set) and is re-used multiple times, typically resulting in over-fitting on this particular set.  

In this paper, we view the learning algorithms as conditional distributions, and 
provide generalization guarantees using the notion of \emph{maximal leakage}. Maximal leakage is a secrecy metric that has appeared both in the computer security literature~\citep{CompSecQuantLeakage,LeakageGeneralized,CompSecAddMultMaxLeakage}, and the information theory literature~\citep{leakage}. It quantifies the leakage of information from a random variable $X$ to another random variable $Y$, and is denoted by $\ml{X}{Y}$. The basic insight is as follows: if a learning algorithm leaks little information about the training data, then it will generalize well. Moreover, similarly to differential privacy, maximal leakage behaves well under composition: we can bound the leakage of a sequence of algorithms if each of them has bounded leakage. It is also robust against post-processing. In addition, it is much less restrictive than differential privacy (e.g., maximal leakage is always bounded for finite random variables, whereas differential privacy can be infinite). As compared to $I_\infty^{\beta}(X;Y)$, which is rarely given in closed form, $\ml{X}{Y}$ is easier to compute and analyze. Indeed, it is simply given by the following formula (for finite $X$ and $Y$): 
\begin{equation} \label{eq:ml1}
\ml{X}{Y}= \log \sum_y \max_{x: P(x)>0} P_{Y|X}(y|x).
\end{equation}

%Following this intuition we derived an approach, using a robust measure of leakage, i.e. maximal leakage \citep{leakage} and connected the generalization capabilities of an algorithm with the amount of information that the hypothesis leaks about the training set. More importantly, using maximal leakage as a dependence measure we were able to provide bounds that generalize the well-known classical ones, derived from the assumption of statistical independence and invalidated by adaptivity. 

\subsection{Overview of Results}
Let $X$ and $Y$ be two random variables distributed over $\X$ and $\Y$, respectively, and let $E \subseteq \X \times \Y$ be an event. Our first result bounds the probability of $E$ under the joint distribution of $(X,Y)$ in terms of the marginal distribution of $X$, the fibers $E_y$ (i.e., $E_y = \{x: (x,y) \in E\}$), and the leakage from $X$ to $Y$:
\begin{align} \label{firstResult}
\mathbb{P}(E) \leq \max_{y\in\mathcal{Y}}\mathbb{P}_X(E_y))\cdot\exp(\ml{X}{Y}).
\end{align}
This type of bound allows us to connect the probability of an event, measured when the dependence is considered, with the probability of the same event but under the assumption of statistical independence (with the same marginals). Whenever we have independence, we have $\ml{X}{Y}=0$. Hence $\exp(\ml{X}{Y})=1$, and we retrieve the bound that does not account for the dependence (examples of these bounds, exploited later on, are McDiarmid's inequality and the significance level used in hypothesis testing to control the false discovery probability). That is, our bounds can \emph{recover} the classical bounds of the non-adaptive setting.
Otherwise, when independence is not satisfied, our bound introduces a multiplicative term that grows as the dependence of $Y$ on $X$ grows. The formal statement and the proof can be found in Section \ref{sec:leakageAdapt}, Theorem \ref{adaptML}.

%The left hand side of the inequality \eqref{firstResult} represents the probability of $E$ when adaptivity is taken into account, while the right hand side is composed of two terms:
%\begin{itemize}
%    \item $\sup_{y\in\mathcal{Y}}\mathcal{P}_X(E_y))$, typically replaced by some upper-bound (e.g. McDiarmid's inequality) and represents the probability of $E$ under the assumption of statistical independence;
%    \item $\exp(\Le(X\to Y))$, a multiplicative term that measures the dependence of $Y$ on $X$. This term is equal to $1$ whenever we have statistical independence leading us back to the non-adaptive case;
%\end{itemize}

With the suitable choice of what $X$, $Y$, and $E$ represent, we derive useful bounds in the contexts of both the generalization error and post-selection hypothesis testing. To wit, our main application is the following bound on the generalization error for $0$-$1$ loss functions.
\begin{mytheorem*}[Theorem \ref{generrML} below] Let $\X$ be the sample space and $\mathcal{H}$ be the set of hypotheses.
	Let $\mathcal{A}:\mathcal{X}^n\to \mathcal{H}$ be a learning algorithm that, given a sequence $S$ of $n$ points, returns a hypothesis $h\in \mathcal{H}$. Suppose $S$ is sampled i.i.d according to some distribution $\mathcal{P}$, i.e., $S\sim \mathcal{P}^n$. 
	Let $E=\{(S,h):|L_{\mathcal{P}}(h)-L_S(h)|>\eta \}$.  Then,
	\begin{align}\label{eq:generrIntro} \mathbb{P}(E) \leq 2\cdot\exp(\ml{S}{\mathcal{A}(S)} -2n\eta^2).\end{align}
\end{mytheorem*}
Theorem~\ref{generrML} offers a bound on the generalization error that can be exponentially decreasing with $n$ since the leakage $\ml{S}{\mathcal{A}(S)}$ cannot grow more than linearly in $n$.  For instance, if $\A$ is $\epsilon$-DP with $\epsilon<2\cdot \eta^2$, we show that the above bound is exponentially decreasing. However, the bound applies more generally even if $\A$ does not satisfy $\epsilon$-differential privacy for any $\epsilon$. More details about this comparison can be found in Section~\ref{sec:leakageAndDP}.

%differently from similar results obtained in~\citep{statValidity, genAdap} the constraint of $\A$ being $\epsilon-$differentially private, for some privacy-parameter $\epsilon$ is not necessary.
%However, whenever $\A$ is $\epsilon-$DP and $\epsilon<2\cdot \eta^2$, Theorem \ref{generrML} guarantees an exponentially decreasing probability that the generalization error of $\A$ is greater than some accuracy $\eta$. 
%COMPARISON WITH THE SAME RESULT OBTAINED FROM DP?\\
%This bound we propose for the generalization error can be exponentially decreasing with $n$. One direct way of accomplishing this is by leveraging, again, on the connection with differential privacy: fixed some $\epsilon$, whenever the learning algorithm is $\alpha$-DP with $\alpha<2\epsilon^2$ our bound guarantees an exponentially decreasing probability that the generalization error of $\mathcal{A}$ is $>\epsilon$.

Another connection between generalization error and information-theoretic measures can be found in~\citep{learningMI} where, under the same assumption, it is shown that $\mathbb{P}(E) \leq O\bigg(\frac{I(S;\A(S))}{n\eta^2}\bigg)$. Whenever our bound is exponentially decreasing in $n$, the improvement in the sample complexity is exponential in the confidence of the bound. This comparison is further discussed in Section~\ref{sec:leakageAndMutI}.

%A simple rewriting of \eqref{firstResult} allows us to compare it with what has been obtained so far with differential privacy. 
Our analysis can be applied for settings more general than the $0$-$1$ loss functions:
\begin{mytheorem*}[Theorem~\ref{generrDPML} below]
	Let $\mathcal{A}:\mathcal{X}^n\to \mathcal{Y}$ be an algorithm. Let $\mathbf{X}$ be a random variable distributed over $\mathcal{X}^n$ and let $Y=\A(\mathbf{X})$. Given $\beta \in [0,1]$, for every $y\in\mathcal{Y}$ let $R(y)\subseteq \mathcal{X}^n$ satisfy $\Prob(\{\mathbf{X}\in R(y)\}) \leq \beta$. Then, 
	\begin{equation}\mathbb{P}(\{\mathbf{X} \in R(Y)\}) \leq \exp(\ml{\mathbf{X}}{Y})\cdot \beta.\end{equation}
\end{mytheorem*}
The idea behind Theorem~\ref{generrDPML} is the following: suppose we have a collection of events $R(y) \subseteq \X^n$ and that each of these events has a small probability of occurring. Then, even when $Y=\A(\mathbf{X})$, the probability that $\mathbf{X}\in R(Y)$ remains small as long as $\ml{\mathbf{X}}{Y}$ is small. This implies that, in the spirit of \citep{statValidity}, given an algorithm $\A$ with bounded maximal leakage, adaptive analyses involving $\A$ can be thought as almost non-adaptive, up to a correction factor equal to $\exp(\Le(\mathbf{X}\to Y))$. 

In \citep[Theorem 11]{statValidity}, under the assumption of $\A$ being $\epsilon$-DP for some range of values of $\epsilon$, it is shown that $\mathbb{P}(\{\mathbf{X} \in R(Y)\}) \leq 3\cdot \sqrt\beta.$ We thus provide a more general result, as the class of algorithms with bounded leakage is not restricted to the differentially private ones. Furthermore, by exploiting a connection between differential and maximal leakage, we show that our bound is tighter if $\epsilon \leq \frac{\log(3/\sqrt{\beta})}{n} \leq \sqrt{\frac{\log(1/\beta)}{2n}}$. Moreover, an immediate consequence of the theorem is that $\mathbb{P}(\{\mathbf{X} \in R(Y)\}) \leq |\Y| \cdot \beta$, since $\ml{X}{Y} \leq \log |\Y|$, thus recovering  the same statement of \citep[Theorem 9]{genAdap}.
The details of the comparisons can be found in Sections~\ref{sec:leakageAndDP} and~\ref{sec:leakageAndMI}.

This rewriting of our results allows us to apply them even in the context of adaptive hypothesis testing.
\begin{mytheorem*}[Theorem \ref{hypTestML} below]
	Let $\mathcal{A} : \mathcal{X}^n \to \mathcal{T}$ be an algorithm
	for selecting a test statistic $t\in\mathcal{T}$. Let $\mathbf{X}$ be a random dataset over $\mathcal{X}^n$. Suppose that $\sigma\in [0, 1]$ is the significance level chosen to control the false discovery probability for the test statistic $t$. Denote
	with $E$ the event that $\mathcal{A}$ selects a statistic such that the null
	hypothesis is true but its p-value is at most $\sigma$ (event that we make a false discovery using $\sigma$ as significance value). Then,
	\begin{equation}\mathbb{P}(E) \leq \exp(\Le(\mathbf{X} \to \mathcal{A}(\mathbf{X}))) \cdot \sigma.\end{equation}
\end{mytheorem*}
%This generalization is particularly useful since existing results using $(\epsilon,\delta)$-DP turn out to provide trivial guarantees~\citep{maxInfo}. 
Similar results based on $I_\infty^\beta(X;Y)$ are derived in~\citep{maxInfo}. However,  $I_\infty^\beta(X;Y)$ requires the knowledge of the prior to be correctly computed. By contrast, to compute $\ml{X}{Y}$ we only need to know the support of $X$, and partial information about the conditional $P_{Y|X}$.
%is also worth noticing that the proof of our theorem is simpler than that of \citep{statValidity} and relies solely on the concept of the $\infty$-order Rényi Divergence.\\

%Another characteristic of Theorem \ref{generrDPML}, that shows how it is actually more general than what stated in \citep{genAdap}, is that, since $\Le(X\to Y)\leq \log|\Y|$, a direct Corollary of Theorem \ref{generrDPML} is: $\mathbb{P}(\{\mathbf{X} \in R(Y)\}) \leq |\Y| \cdot \beta$, exactly 

%Furthermore, one of the shortcomings of differential privacy is that the guarantees it provides are connected to the sensitivity of the statistics in consideration. When applied to the problem of adaptively performing hypothesis tests while generating statistically valid $p$-values it provides trivial guarantees \citep{maxInfo}. Our results allows us to provide some guarantees with very strong assumptions on the maximal leakage, i.e. it has to be bounded by a constant. More details can be found in Section \ref{sec:leakageAndDP}. On the other hand, in the same framework, with a specific choice of the set $E$ it is possible to provide guarantees on the probability of committing a false discovery even in these adaptively defined hypothesis tests:

Other than providing generalization guarantees, one important characteristic of maximal leakage is that it is robust under post-processing and composes adaptively.
\begin{mylemma*}[Robustness to post-processing, Lemma \ref{robustnessLeakage} below]
    Let $\mathcal{A}:\mathcal{X}\to \mathcal{Y}$ and $\mathcal{B}:\mathcal{Y}\to\mathcal{Y'}$ be two algorithms. Then,
	$\ml{X}{\B(\A(X))} \leq \ml{X}{\A(X)}$.
\end{mylemma*}
The implication is the following: no matter which post-processing you apply to the outcome of your algorithm, the generalization guarantees provided by having a bounded leakage on $\A(X)$ will extend to $\B(\A(X))$ as well, making it a robust measure.
\begin{mylemma*}[Adaptive composition, Lemma \ref{thCompositionN} below]
	Consider a sequence of $n\geq 1$ algorithms: $(\mathcal{A}_1,\ldots,\mathcal{A}_n)$ such that for every $k\leq n$ the outputs of $\A_1,\ldots,\A_{k-1}$ are inputs to $\A_k$. Then, denoting by $A_1,\ldots, A_n$ the (random) outputs of the algorithm:
	\begin{equation} \Le(X\to (A_1,\ldots,A_n)) = \Le(X\to \textbf{A}^n) \leq \sum_{i=1}^n k_i. \end{equation}
\end{mylemma*}
From this we can derive that if each of the $\A_i$ has a bounded leakage (and thus, generalizes), even the adaptive composition of the whole sequence will have bounded leakage (although, with a worse bound) and potentially maintain the generalization guarantees and avoid overfitting. This property is fundamental for practical applications.

%Once again, if $\mathcal{A}$ is independent from $\mathbf{X}$, we recover the classical bound of $\sigma$. 
%A similar result is derived in \citep{maxInfo} using the concept of $\beta$-approximate max-information. The connection between the two measures is further discussed in Section \ref{sec:leakageAndMI}, but we note that it is unclear ho to calculate $\beta$-approximate max-information. By contrast, while maximal leakage can be easily computed once the conditional distributions $Y|X=x$ are known \textit{and} without the necessity of knowing the marginal distribution of $X$, it is rather unclear how to compute the $\beta$-approximate max-information between two random variables.  

%\paragraph{Outline:} In Section~\ref{sec:leakage} we review the properties of maximal leakage and show that, just like differential privacy and max-information, it is robust under post-processing and composes adaptively. This implies that whenever we have an adaptively composed finite sequence of algorithms such that each of them has bounded leakage, the adaptive composition of the sequence will still have bounded leakage. Furthermore, thanks to the generalization guarantees it provides, not only every element of the sequence will generalize well, but the adaptive composition will maintain such property, having a bounded leakage.\\

\subsection{More Related Work}
% to Amedeo: references to papers on short description length, is it also Dwork et al.?
In addition to differentially private algorithms, Dwork \emph{et al.}~\citep{genAdap} show that algorithms whose output can be described concisely generalize well. They further introduce $\beta$-max information to unify the analysis of both classes of algorithms. Consequently, one can provide generalization guarantees for a sequence of algorithms that alternate between differential privacy and short description. In~\citep{maxInfo}, the authors connect $\beta$-max information with the notion of approximate differential privacy, but show that there are no generalization guarantees for an arbitrary composition of algorithms that are approximate-DP and algorithms with short description length.

With a more information-theoretic approach, bounds on the exploration bias and/or the generalization error are given in~\citep{explBiasMI,infoThGenAn,jiao2017dependence,explBiasLeak}, using mutual information and other dependence-measures. 
\subsection{Outline}
The work is organized in the following way. In Section \ref{probSett} we will discuss the problem setting. We will properly define Adaptive Data Analysis in Section \ref{ada} and Maximal Leakage in Section \ref{ml}, along with the presentation of most of its properties (including robustness to post-processing and adaptive composition). 
In Section \ref{main} we will present our main results: the generalization guarantees implied by a bound on the maximal leakage (Section \ref{sec:leakageAdapt}), its application to post-selection hypothesis testing (Section \ref{sec:psht}), we will also compare the measure to Differential Privacy, Mutual Information and Max-Information (Sections \ref{sec:leakageAndDP}, \ref{sec:leakageAndMutI}, \ref{sec:leakageAndMI}).
For the proofs the reader is referred to the Appendices.
\section{Problem setting}\label{probSett}
\subsection{Adaptive Data Analysis}\label{ada}
Before entering into the details of our results, we define the model of adaptive composition we will be considering throughout the exposition and already used in~\citep{statValidity,genAdap,maxInfo}. 
\begin{mydef}
	Let $\X$ be a set. Let $S$ be a random variable over $\X^n$.
	Let $(\mathcal{A}_1,\ldots,\mathcal{A}_m)$ be a sequence of algorithms such that $\forall i: 1\leq i \leq m\quad  \mathcal{A}_i  : \mathcal{X}^n \times \mathcal{Y}_1 \times \ldots \times  \mathcal{Y}_{i-1} \to \mathcal{Y}_i$.
	Denote with $Y_1 =\A_1(S), Y_2=\A_2(S,Y_1), \ldots, Y_m = \A_m(S,Y_1,\ldots,Y_{m-1})$.
	%for every $1\leq i,\leq m$ the input of $\A_i$ consists of $S$ and the sequence $(y_1,\ldots,y_{i-1})$ of outputs of the previously executed algorithms $\A_1,\ldots, \A_{i-1}$.
	The adaptive composition of $(\mathcal{A}_1,\ldots,\mathcal{A}_m)$ is an algorithm that takes as an input $S$ and sequentially executes the algorithms $(\A_1,\ldots,\A_m)$ as described by the sequence $(Y_i, 1\leq 1 \leq m)$
\end{mydef}ß
%In this framework, a dataset is viewed as a random sequence $S$ of $n$ points, drawn  from the domain $\mathcal{X}^n$, according to a certain distribution $\mathcal{D}$. The dataset will be considered to be fixed at the beginning, and will be the same for every algorithm in the sequence. Each $\mathcal{A}_i$ takes as input $S$ \textit{and} the sequence of outputs $(y_1,\ldots, y_{i-1})$ of the previously executed algorithms $\mathcal{A}_1,\ldots,\mathcal{A}_{i-1}$: this allows us to model the adaptivity of the analysis. Furthermore, t
This level of generality allows us to formalize the behaviour of a data analysts who, after viewing the previous outcomes of the analysis performed, decides what to do next. A potential analyst would execute a sequence of algorithms that are known to have a certain property (e.g. generalize well) when used without adaptivity. The question we would like to address is the following: is this property also maintained by the adaptive composition of the sequence?
The answer is not trivial as, for every $i$, the outcome of $\mathcal{A}_i$ depends both on $S$ and on the previous outputs,  that depend on the data themselves. However, when this property is guaranteed by some measure that composes adaptively itself (like differential privacy or, as we will show soon, maximal leakage) then it can be preserved.

For the remainder of this paper, we will only consider finite sets, and $\log$ is taken to the base $e$.

\subsection{Maximal Leakage}\label{ml}
\label{sec:leakage}
In this section, we review some basic properties of maximal leakage. As mentioned earlier, the main properties of (approximate) differential privacy and $\beta$-max information that make them useful for adaptive data analysis are: 1) robustness to post-processing, and 2) adaptive composition. 
We show that maximal leakage also satisfies these properties, with the advantage of being less restrictive than differential privacy, and easier to analyze than $I_\infty^{\beta}(X;Y)$. 
$\ml{X}{Y}$ was introduced as a  way of measuring the leakage from $X$ to $Y$, hence the following definition:

%Taking inspiration from the state of the art in this area we propose a new line of work based on another secrecy-related measure known as maximal leakage \citep{leakage}. Similarly to approximate max-information, the maximal leakage has been proposed as an Information Theoretic measure of dependence and is derived from the Rényi Divergence of order infinity. 
%It has been  order to answer to the following question: let $X$ be a random variable representing sensitive information and let $Y$ be the observation of $X$ through a side channel, how much information does $Y$ \textbf{leak} about $X$?\\
%Let us assume that the random variables at play are discrete.
\begin{mydef}[Def. 1 of \citep{leakage}]\label{leakage}
	Given a joint distribution $P_{XY}$ on finite alphabets $\mathcal{X}$ and $\mathcal{Y}$, the maximal leakage from $X$ to $Y$ is defined as:
	\begin{equation}  \mathcal{L}(X\to Y)= \sup_{\substack{U-X-Y-\hat{U}}} \log \frac{\mathbb{P}(\{U=\hat{U}\})}{\max_{u\in\mathcal{U}} \mathbb{P}_U(\{u\})},\end{equation}
	where $U$ and $\hat{U}$ take values in the same finite, but arbitrary, alphabet. 
\end{mydef} \noindent
It is shown in~\citep[Theorem 1]{leakage} that: 
\begin{align}\mathcal{L}(X\to Y) = \log \sum_{y\in \mathcal{Y}} \max_{\substack{x\in\mathcal{X}}: P_X(x)>0} P_{Y|X}(y|x). \end{align}
%We can provide bounds on the generalization capabilities of Learning Algorithms analyzing how much Information is being leaked from the Dataset through the Hypothesis. Moreover, this measure, similarly to the others, can be composed adaptively and is robust under post-processing, allowing us to use it in the framework of Adaptive Data Analysis.
%Also, important properties of the maximal leakage will be now listed. 
Some important properties of the maximal leakage are the following: 

\begin{mylemma}[\citep{leakage}]
	\label{leakageProps}
	
	For any joint distribution $P_{XY}$ on finite alphabets $\X$ and $\Y$, \begin{enumerate}
		\item $\ml{X}{Z} \leq \min\{\ml{X}{Y},\ml{Y}{Z}\}$, if the Markov chain $X - Y - Z$ holds.
		\item $\ml{X}{Y} \geq 0$, with equality if and only if $X$ and $Y$ are
		independent.
		%\item If $\{(X_i
		%, Y_i), 1\leq i \leq n\}$ are mutually independent, then
		%$\ml{X^n}{Y^n} = \sum_{i=1}^n \ml{X_i}{Y_i}$.
		\item $\ml{X}{Y} \leq \min\{\log |X |, \log |Y|\}$.
		\item $\ml{X}{Y} \geq I(X; Y )$.
	\end{enumerate}
\end{mylemma}

%Moreover, two important properties that a measure needs to satisfy, in order to be fruitfully used in Adaptive Data Analysis are the Robustness to post-processing and the Adaptive Composability. The relevance is due to the fact that, given a sequence of $n$ algorithms, each of which has good generalization capabilities, it is unclear whether or not they retain this capabilities when adaptively composed. When the measure that ensures this generalization capability can be composed adaptively instead, it is possible to analyze how and if this property is maintained. \\
%Differential Privacy and Max-Information both satisfy these properties and the details are exposed in Appendix \ref{app:definitions}.
%Maximal Leakage also enjoys these two properties and the proofs are presented in Appendix \ref{app:leakageProps}. 
\noindent The following is a direct application of the first property of the lemma:
\begin{mylemma}[Robustness to post-processing]\label{robustnessLeakage}
	Let $\mathcal{X}$ be the sample space and let $X$ be distributed over $\X$. Let $\Y$ and $\mathcal{Y'}$ be output spaces, and consider $\mathcal{A}:\mathcal{X}\to \mathcal{Y}$ and $\mathcal{B}:\mathcal{Y}\to\mathcal{Y'}$. Then,
	$\ml{X}{\B(\A(X))} \leq \ml{X}{\A(X)}$.
\end{mylemma}
The useful implication of this result is as follows: any generalization guarantees provided by $\A$ cannot be invalidated by further processing the output of $\A$. 
In order to analyze the behavior of the adaptive composition of algorithms in terms  of maximal leakage, we first need the following definition of conditional maximal leakage
\begin{mydef}[Conditional Maximal Leakage \citep{leakageLong}]
	Given a joint distribution $P_{XYZ}$ on alphabets $\mathcal{X}, \mathcal{Y}, \text{ and } \mathcal{Z}$, define:
	\begin{equation} \Le (X \!\! \to \!\! Y |Z) =\sup_{\substack{U:U-X-Y|Z}} \log\frac{\mathbb{P}(\{U = \hat{U}(Y, Z)\})}{\mathbb{P}(\{U=\tilde{U}(Z)\})},\end{equation}
	where $U$ takes value in an arbitrary finite alphabet and we consider $\hat{U}, \tilde{U}$ to be the optimal estimators of $U$ given $(Y,Z)$ and $Z$, respectively.
\end{mydef}
\noindent Again, it is shown in~\citep{leakageLong} that:
\begin{align} \Le (X \!\! \to \!\! Y |Z) & = \log \bigg( \max_{\substack{z:P_Z(z)>0}} \sum_{y} \max_{\substack{x: P_{X|Z}(x|z)}>0} P_{Y|XZ}(y|xz)\bigg),  \label{conditionalLeakage} \\
\text{and } \ml{X}{(Y,Z)} & \leq \ml{X}{Y} + \Le (X \!\! \to \!\! Z |Y). \label{ineqLeak}
\end{align}
A useful consequence of this inequality is the following result. 
\begin{mylemma}[Adaptive Composition of Maximal Leakage]\label{thComposition2}
	Let $\mathcal{A}:\mathcal{X}\to \mathcal{Y}$ be an algorithm such that $\Le(X\to \mathcal{A}(X)) \leq k_1$. 
	Let $\mathcal{B}:\mathcal{X}\times \mathcal{Y}\to \mathcal{Z}$ be an algorithm such that for all $ y \in \mathcal{Y},~ \Le(X\to\mathcal{B(}X,y))\leq k_2$. 
	Then $\Le \Big(X\to \big(\mathcal{A(}X),\mathcal{B}(X,\mathcal{A}(X))\big)\Big) \leq k_1+k_2$.
\end{mylemma}

%An important property of the maximal conditional leakage is the following: let $X,Y=\mathcal{A}(X),Z=\mathcal{B}(X,Y)$ represent the random variables, modelling the randomness respectively in the samples, output of $\mathcal{A}$ and output of $\mathcal{B}$.\\ The following holds \citep{leakage}:
%	\begin{align}\Le(X\to(Y, Z)) \leq \Le (X\to Y) + \Le (X\to Z|Y).\end{align}
\noindent The proof of this lemma relies crucially on the fact that maximal leakage depends on the marginal $P_X$ only through its support. 
In order to generalize the result to the adaptive composition of $n$ algorithms, we need to lift the property stated in the inequality \eqref{ineqLeak} to more than $2$ outputs.
\begin{mylemma}\label{compositionN}
	Let $n\geq1$ and $X,A_1,\ldots,A_n$ be random variables.
	\begin{equation} \Le(X\to (A_1,\ldots,A_n)) \leq \Le(X\to A_1) + \Le(X\to A_2|A_1)+ \ldots + \Le(X\to A_n| (A_1,\ldots,A_{n-1})).\end{equation}
\end{mylemma} \noindent 
An immediate application of Lemma \ref{compositionN} % together with the same observation made above, i.e. $\forall 1\leq j \leq n,\, \mathbb{P}(\{A_j(X,A_1^{j-1}) = a_j\}) = \mathbb{P}(\{A_j=a_j\}| \textbf{A}^{j-1}, X)$, 
leads us to the following result.
\begin{mylemma}\label{thCompositionN}
	Consider a sequence of $n\geq 1$ algorithms: $(\mathcal{A}_1,\ldots,\mathcal{A}_n)$ where for each  $ 1\leq i \leq n$, $\mathcal{A}_i : \mathcal{X}\times \mathcal{Y}_1\times\ldots \times \mathcal{Y}_{i-1} \to \mathcal{Y}_i$. Suppose that for all $ 1\leq i \leq n$ and for all  $(y_1,\ldots,y_{n-1})\in \mathcal{Y}_1\times\ldots \times \mathcal{Y}_{i-1}$ , $\Le(X\to \mathcal{A}_i(X,y_1,\ldots,y_{i-1})) \leq k_i$. Then, denoting by $A_1,\ldots, A_n$ the (random) outputs of the algorithm:
	\begin{equation} \Le(X\to (A_1,\ldots,A_n)) = \Le(X\to \textbf{A}^n) \leq \sum_{i=1}^n k_i. \end{equation}
\end{mylemma}
\noindent The proofs can be found in Appendix \ref{app:leakageProps}.
\section{Main Results}\label{main}

Inspired by the results based on differential privacy~\citep{learningDp} and mutual information~\citep{learningMI}, we derive new bounds using maximal leakage. 
The underlying intuition is the following: if the outcome of a learning algorithm leaks only little information from the training data, then it will generalize well. In addition to adaptive composition and robustness against post-processing, maximal leakage has the following advantages:
\begin{itemize}
	\item it can be computed using only a high level description of the conditional distributions $P_{Y|X=x}$ and depends on $P_X$ only through the support;
	\item it allows us to obtain an exponentially decreasing bound in the number of samples of the training set.
\end{itemize}

\subsection{Low Leakage implies low generalization error}
\label{sec:leakageAdapt}

Our main result allows us to connect the probability of some event happening under the assumption of statistical independence between the input and output random variables $X$ and $Y$ and the probability of the same event, taking into account the dependence of the output from the input.
This connection relies on the R\'enyi divergence of order infinity between the distributions of $X$ and $X|Y$, the first one representing the statistical independence scenario and the second one considering the dependence between the two. (In~\citep{learningMI}, a similar connection is considered, using KL divergence instead.)
\begin{mytheorem}\label{adaptML}
	Let $\mathcal{P}$ be a distribution on the space $\mathcal{X}\times\mathcal{Y}$ and denote with $\mathcal{P}_X$ the marginal distribution of the random variable $X$. Let $E$ be an event and for every $y \in \mathcal{Y}$, denote with $E_y$ the set $\{x:(x,y)\in E\}$. Then,
	%If $E$ is such that $\forall y \in \mathcal{Y} \,\, \mathcal{P}_X(E_y)\leq \beta$, then
	\begin{equation} \mathcal{P}(E) \leq \exp({\Le(X\to Y)})\cdot \max_{y \in \Y} \mathcal{P}_X(E_y) .\end{equation}
\end{mytheorem}
The proof is in Appendix \ref{app:proof}.
An immediate application can be found in statistical learning theory and relies on McDiarmid's inequality. More precisely, the object of study is the generalization error of learning algorithms. To wit: 
%we have a domain set  $\mathcal{X}$, the set of objects that we may wish to label, and a label set $\mathcal{Y}$, the set of possible labels. The learner is an algorithm that, given as an input a finite sequence of $n$ data-label pairs $(z_1,\ldots,z_n) = ((x_1,y_1),\ldots,(x_n,y_n))$, outputs a prediction rule $h$, typically known as hypothesis or classifier and used to assign labels to unseen data points. Formally:
\begin{mydef}
	Let $\mathcal{X},\mathcal{Y}$ be two sets, respectively the domain set and the label set and let $\mathcal{Z}=\mathcal{X}\times\mathcal{Y}$. Let $\mathcal{H}$ be the hypothesis class, i.e. a set of prediction rules $h:\mathcal{X}\to\mathcal{Y}$. Given $n\in\mathbb{N}$, a learning algorithm $\mathcal{A}$ is a map $\mathcal{A}:\mathcal{Z}^n\to \mathcal{H}$ that, given as an input a finite sequence of domain points-label pairs $S=((x_1,y_1),\ldots,(x_n,y_n))$ , outputs some classifier $h=\mathcal{A}(S)\in\mathcal{H}$.
\end{mydef} \noindent
In order to estimate the capability of some learning algorithm of correctly classifying unseen instances of the domain set i.e. its \emph{generalization} capability, the concept of generalization error is introduced.
\begin{mydef} Let $\mathcal{P}$ be some distribution over $\mathcal{Z}=\mathcal{X}\times\mathcal{Y}$. The error (or risk) of a prediction rule $h:\mathcal{X}\to\mathcal{Y}$ with respect to $\mathcal{P}$ is defined as \begin{equation}L _\mathcal{P}(h)=\mathbb{E}_{(x,y)\sim \mathcal{P}}(\mathbbm{1} (h(x) \neq y))=\mathbb{P}_{(x,y)\sim \mathcal{P}}(h(x)\neq y),\end{equation}
	while, given a sample $S=((x_1,y_1),\ldots,(x_n,y_n))$, 
	the empirical error
	of $h$ with respect to $S$ is defined as \begin{equation}\label{empRisk}L_{S}(h) = \frac1n \sum_{i=1}^n \mathbbm{1}(h(x_i)\neq y_i).\end{equation}
	Moreover, given a learning algorithm $\mathcal{A}:\mathcal{Z}^n\to\mathcal{Y}$, its generalization error with respect to $S$ is defined as: \begin{equation}\label{generr}\text{gen-err}_\mathcal{P}(\mathcal{A},S)=|L_{\mathcal{P}}(\mathcal{A}(S))-L_{S}(\mathcal{A}(S))|.\end{equation}
\end{mydef}
%\iffalse
%\begin{mydef}[The realizability assumption]
%Given a hypothesis class $\mathcal{H}$ and a distribution $\mathcal{P}$ over the data, there exists $h^*\in %\mathcal{H}$ s.t. $L_{\mathcal{P}}(h^*) = 0$.
%\end{mydef}
%\begin{mydef}[PAC learnability]
%A hypothesis class $\mathcal{H}$ is PAC learnable if there exist a function $m_{\mathcal{H}} : (0,1)^2 \to \mathbb{N}$ and a learning algorithm with the following property: For every $\epsilon, \delta \in (0, 1)$, for every distribution $\mathcal{P}$ over $\X$ if the realizable assumption holds with respect to $\mathcal{H},\mathcal{P}$, then when running the learning algorithm on $m \geq m_{\mathcal{H}}\epsilon,\delta)$ i.i.d. examples generated by $\mathcal{P}$ the algorithm returns a hypothesis $h$ such that, with probability of at least $1 − \delta$ (over the choice of the examples), $L_{\mathcal{P}}(h) \leq \epsilon$.
%\end{mydef}
%\fi
\begin{mytheorem}\label{generrML}
	Let $\X \times \Y$ be the sample space and $\mathcal{H}$ be the set of hypotheses.
	Let $\mathcal{A}:\mathcal{X}^n \times \mathcal{Y}^n\to \mathcal{H}$ be a learning algorithm that, given a sequence $S$ of $n$ points, returns a hypothesis $h\in \mathcal{H}$. Suppose $S$ is sampled i.i.d according to some distribution $\mathcal{P}$ over $\X \times \Y$, i.e., $S\sim \mathcal{P}^n$. 
	Given $\eta \in (0,1)$, let $E=\{(S,h):|L_{\mathcal{P}}(h)-L_S(h)|>\eta \}$.  Then,
	\begin{align}\mathbb{P}(E) \leq 2\cdot\exp(\ml{S}{\mathcal{A}(S)} -2n\eta^2).\end{align}
\end{mytheorem}

\begin{proof}
	Let us denote with $E_h$ the fiber of $E$ over $h$. By McDiarmid's inequality (Lemma \ref{mcdiarmids}), with $c=1/n$, we have that for every hypothesis $h\in\mathcal{H},\, \mathcal{P}_S(E_h) \leq 2\cdot \exp(-2n\eta^2)$, as the empirical error defined in equation \eqref{empRisk} has sensitivity of $1/n$. Then it follows from Theorem~\ref{adaptML} that:
	\begin{equation}\mathbb{P}(E) \leq \exp(\Le(S\to \mathcal{A}(S)))\cdot 2\exp(-2n\eta^2),\end{equation} 
	as desired.
	\iffalse
	Thanks to Bayes Theorem we have that:
	$$ \mu(E) = \mu_{h|S}(E_S) \cdot \mu_S(E_h)$$
	Using McDiarmid's inequality, we can say that:
	$$\mu_S(E_h)\leq 2\cdot\exp(-2n\epsilon^2)$$
	\begin{align}
	\mu(E) &= \mu_{h|S}(E_S) \cdot \mu_S(E_h)\\  
	&\leq \mu_{h|S}(E_S) \cdot   2\cdot\exp(-2n\epsilon^2) \\ 
	&\leq \sup_S(\mu_{h|S}(E_S)) \cdot 2\exp(-2n\epsilon^2) \\
	& \leq  \exp(\mathcal{L}(S\to\mathcal{A}(S)) \cdot 2\exp(-2n\epsilon^2)
	\end{align}
	Where $(4)$ follows from the following inequalities:\\
	$\sup_S(\mu_{h|S}(E_S)) = \sup_S(\sum_{h\in E_S} \mu_{h|S}(h)) \leq \sum_{h\in E_S}(\sup_S(\mu_{h|S}(h))) \leq \sum_{h\in\mathcal{H}} (\sup_S(\mu_{h|S}(h)))$\\$ =  \exp(\mathcal{L}(S\to\mathcal{A}(S))$
	\fi
\end{proof}
Whenever $\mathcal{A}$ is independent from the samples $S$ we have that $\exp(\mathcal{L}(S\to\mathcal{A}(S)))=1$ and we immediately fall back to the non-adaptive scenario. We are hence proposing a generalization of these bounds whenever adaptivity is introduced. A concrete example can be seen with Theorem \ref{generrML}: in this particular case, whenever $\mathcal{A}(S)$ is independent from $S$ we immediately retrieve that
$\mathbb{P}(E)\leq 2\cdot \exp(-2n\eta^2)$ i.e. McDiarmid's inequality with sensitivity $1/n$.\\
%Moreover, the measure used in the bound  has the following advantage: given only a high level description of the distributions $P_{Y|X=x}$ it is possible to compute the Leakage of Information from the dataset through the Hypothesis and, let $P_X$ be fixed, then $\exp{\Le(X\to Y )}$ is convex in $P_{Y|X}.$ Hence, minimizing the Maximal Leakage subject to convex constraints is a convex minimization problem and it is possible, at least in principle, to efficiently find the randomized algorithm that minimizes the measure. Furthermore it depends on $P_X$ only through the support.
\subsection{Maximal Leakage and Differential Privacy}
\label{sec:leakageAndDP}

In the line of work started by Dwork et al. \citep{statValidity,genAdap}, the idea is to exploit the stability induced by differential privacy in order to derive generalization guarantees. The notion of algorithms that we will be considering is slightly more general than the notion of learning algorithms exposed before. 
Before proceeding to the comparison let us state a simple, 
but useful relationship between maximal leakage and pure differential privacy, proved in Appendix~\ref{app:leakageProps}.
\begin{mylemma}\label{MLandDP}
	Let $\mathcal{A}:\mathcal{X}^n\to\mathcal{Y}$ be an $\epsilon$-Differentially Private randomized algorithm, then $\Le(X\to\mathcal{A}(X))\leq \epsilon\cdot n$.
\end{mylemma} \noindent
This suggests an immediate application of Theorem \ref{generrML}. Indeed, suppose $\mathcal{A}$ is an $\epsilon$-DP algorithm, then:
\begin{align} \label{eq:expbound}
\exp(\Le(X\to Y)-2n\eta^2) \leq \exp(\epsilon n -2n\eta^2) = \exp(-n(2\eta^2-\epsilon))
\end{align}
In order for the bound to be decreasing with $n$, we need $2\eta^2-\epsilon>0$ leading us to $\epsilon<2\cdot \eta^2$, where $\eta$ represents the accuracy of the generalization error and $\epsilon$ the privacy parameter. Thus, for fixed $\eta$, as long as the privacy parameter is smaller than $2\cdot \eta^2$, we have guaranteed generalization capabilities for $\mathcal{A}$ with an exponentially decreasing bound. For $\epsilon \leq \eta/2$, it is shown in~\citep[Theorem 9]{statValidity} that $\mathbb{P}(E) \leq \frac14 \exp{\left(-n\eta^2/12\right)}.$ It is easy to check that, for large enough $n$, our bound is tighter if $\epsilon \leq \frac{23}{12}\eta^2$.
A more general result shown in~\citep{statValidity} is the following:
\begin{mytheorem}[Thm. 11 of \citep{statValidity}]\label{generrDP}
	Let $\mathcal{A}:\mathcal{X}^n\to \mathcal{Y}$ be an $\epsilon-$differentially private algorithm. Let $X\sim\mathcal{P}^n$ be a random variable. Let $Y=\mathcal{A}(X)$ be the corresponding output random variable. Assume that for every $y\in\mathcal{Y}$ there is a subset $R(y)\subseteq \mathcal{X}^n$ such that $\max_{y\in\mathcal{Y}} \Prob(\{\mathbf{X}\in R(y)\}) \leq \beta$. Then, for $\epsilon \leq \sqrt{\frac{\ln(1/\beta)}{2n}}$ we have that \begin{equation}\Prob(\{\mathbf{X}\in R(Y)\}) \leq 3\sqrt{\beta}.\end{equation}
\end{mytheorem} \noindent
The theorem just stated shows that, if we have a collection of events $\{R(y)| R(y)\subseteq \mathcal{X}^n, y \in \mathcal{Y}\}$ and each of these has a small probability of happening, then even considering $R(Y)$, with $Y=\mathcal{A}(X)$, i.e. introducing adaptivity, the probability will remain small. 
This theorem, when applied together with McDiarmid's inequality allows us to characterize the generalization capabilities of DP algorithms, by simply using as $\beta$ the right-hand side of McDiarmid's inequality.
%\iffalse
%Together with McDiarmid's inequality this theorem leads us to an interesting characterization of the generalization capabilities of DP algorithms:
%\begin{mycorollary}[Corollary 7 of \citep{genAdap}]
%	\label{genPureDiff}
%	Let $\mathcal{A}$ be an algorithm that outputs a c-sensitive function $f:\mathcal{X}^n\to\mathbb{R}$. Let $\mathbf{X}$ be a random dataset chosen according to the distribution $\mathcal{P}^n$ over $\mathcal{X}^n$ and let $f=\mathcal{A}(\mathbf{X})$. If $\mathcal{A}$ is $\frac{\tau}{cn}$-differentially private then $$\mathbb{P}(f(\mathbf{X})-\mathcal{P}^n(f) \geq \tau) \leq 3\exp(-\tau^2/(c^2n)).$$
%\end{mycorollary}
%$\mathcal{P}^n$ represents a product distribution and $\mathcal{P}^n(f)$ represents the expected value of $f$.\fi
A rephrasing of Theorem~\ref{adaptML} allows us to compare our results with Theorem~\ref{generrDP}:
\begin{mytheorem}\label{generrDPML}
	Let $\mathcal{A}:\mathcal{X}^n\to \mathcal{Y}$ be an algorithm. Let $\mathbf{X}$ be a random variable distributed over $\mathcal{X}^n$ according to $\mathcal{P}^n$ and let $Y=\A(\mathbf{X})$. Assume that for every $y\in\mathcal{Y}$ there is a subset $R(y)\subseteq \mathcal{X}^n$ with the property that $\max_{y\in\mathcal{Y}} \Prob(\{\mathbf{X}\in R(y)\}) \leq \beta$. Then, \begin{equation} \mathbb{P}(\{\mathbf{X} \in R(Y)\}) \leq \exp({\Le(\mathbf{X}\to Y)})\cdot \beta.\end{equation}
\end{mytheorem} \noindent
The proof is immediate once we have Theorem \ref{adaptML}.
%We can see, again, that when $\mathbf{X}$ is independent from $Y$ we retrieve the same bound we started with, that controls the probability of the event \textit{without} adaptivity.\\ %On the other hand, when we do not have independence, %our bound allows us to tell how far is the probability of the set $R(Y)$, where $Y$ encodes the dependency of $\mathcal{A}$ from $\mathbf{X}$, with the probability 
%measuring the dependence introduced by $\mathcal{A}$ via the Maximal Leakage, we are able to see how much the probability of the set $R(Y)$ grows with respect to $R(y)$. 
The type of result we are providing here is qualitatively different from the ones derived with differential privacy. We do not pose any constraint on the algorithm itself but rather propose a way of estimating how the probabilities we are interested in change, by measuring the level of dependence we are introducing using maximal leakage. 

%\begin{proof}
%Let $Y=\mathcal{A}(X)$. Fix some $\hat{\mathbf{x}}\in\mathcal{X}^n$, $\forall\,\mathbf{x}\in\mathcal{X}^n$ we have that $d_H(\mathbf{x},\hat{\mathbf{x}})\leq n$ and, iteratively applying the definition of Differential Privacy, we have that $\mathbb{P}(Y=y|X=\mathbf{x})\leq e^{\epsilon\cdot n} \mathbb{P}(Y=y|X=\hat{\mathbf{x}}).$\\
%\begin{align}\Le(X\to Y) &= \log \sum_{y\in\mathcal{Y}} \max_{\mathbf{x}\in\mathcal{X}^n} \mathbb{P}(Y=y|X=\mathbf{x})\\ &\leq \log\sum_{y\in\mathcal{Y}} e^{\epsilon\cdot n}\mathbb{P}(Y=y|X=\hat{\mathbf{x}}) \\ &= n\cdot\epsilon \end{align}
%\end{proof}
Now, suppose we have an $\epsilon$-differentially private algorithm with $\epsilon < \frac{\log{(3/\sqrt{\beta})}}{n}\leq \sqrt{\frac{\ln(1/\beta)}{2n}} $. Theorem~\ref{generrDP} provides a \emph{fixed} bound of $3\sqrt{\beta}$, while with our Theorems~\ref{generrDPML} and~\ref{MLandDP} we obtain that:
\begin{equation}
\exp({\Le(\mathbf{X}\to Y)})\cdot \beta < \exp{(\log{(3/\sqrt{\beta}))}}\cdot \ \beta = 3\sqrt{\beta}.
\end{equation}
Hence, whenever the privacy parameter is lower than $\frac{\log{(3/\sqrt{\beta})}}{n}$ we are able to provide a better bound. %It is also worth pointing out that Theorem \ref{MLandDP} gives a \emph{worst-scenario} type of bound, giving us a very pessimistic view of the amount of information Leaked. Consider, for instance, a conditional distribution $P_{Y|X}$ uniform-like: we have that the ratio between probabilities is close to $1$ and can thus be much lower than $e^{\epsilon\cdot n}$.\\
In the spirit of comparison with the differential privacy-derived results, let us also state Theorem \ref{generrML} with a general sensitivity $c$:
\begin{align} \label{eq:withsensitivity}
\mathbb{P}(E)\leq 2\cdot \exp\bigg(\Le(\mathbf{X}\to Y)- \frac{2\eta^2}{c^2n}\bigg)
\end{align}

Corollary 7 of \citep{genAdap} states that whenever an algorithm $\A:\X^n\to \Y$ outputs a function $f$ of sensitivity $c$ and is $\eta/(cn)-$DP then, denoting with $S$ a random variable distributed over $\X^n$ and with $E=\{(S,f): f(S)-\E(f)\geq \eta\}$ we have that $\mathbb{P}(E)\leq 3\exp(-\eta^2/(c^2n))$.
It is easy to see that Theorem~\eqref{eq:withsensitivity} provides a tighter bound whenever the accuracy $\eta> n\cdot c$.
%Again, this bound can be very loose and possibly improvable.
%\section{Maximal Leakage and Adaptive Composition}
As already stated before, the family of algorithms with bounded maximal leakage is not restricted to algorithms that are differentially private: a simple example can be found in algorithms with a bounded range, since $\Le(X\to Y)\leq \min\{\log|\X|,\log|\Y|\}$. This simple relation allows us to immediately retrieve another result stated in \citep[Theorem 9]{genAdap}: $\mathbb{P}(\{\mathbf{X} \in R(Y)\}) \leq |\Y|\cdot \beta$, showing how Theorem \ref{generrDPML} is more general than both Theorems 6 and 9 of~\citep{genAdap}.

\subsection{Post-selection Hypothesis Testing}\label{sec:psht}
As already underlined in \citep{maxInfo}, most of the state of the art results that rely on differential privacy provide interesting guarantees only when we are considering low-sensitivity functions. They cannot be applied, for instance, to the problem of adaptively performing hypothesis tests while providing statistical guarantees. The reason for this is that $p-$values have a sensitivity larger that $0.37/\sqrt{n}$ \citep[Lemma B.1]{maxInfo}. The theorem proven in \citep{algoStability} that relies on $(\epsilon,\delta)-$DP, when applied to $p-$values, provides trivial error guarantees. 
When applying Theorem \ref{generrML} to p-values, the result we get is the following:
\begin{equation}\mathbb{P}(E) \leq \exp (\Le(\mathbf{X} \to Y) -2\eta^2), \end{equation} 
suggesting that, in order to get a meaningful bound, according to this measure of leakage, it is necessary to leak very little information about the dataset. More precisely, we need $\Le(\mathbf{X} \to Y)\leq 2\eta^2$. Notice that in the same framework but without adaptivity, McDiarmid's inequality itself is not able to provide a better bound than $\exp(-2\eta^2)$.

Consider instead the problem of bounding the probability of making a false discovery, when the statistics to apply is selected with some data dependent algorithm. In this context, the guarantees that allow to upper-bound this probability by the significance value no longer hold. It is possible to show, using maximal leakage, how to adjust the significance level in these adaptive settings, in order to have a guaranteed bound on the probability of error.
As a Corollary of Theorem \ref{adaptML} we can retrieve the following:
\begin{mytheorem}\label{hypTestML}
	Let $\mathcal{A} : \mathcal{X}^n \to \mathcal{T}$ be a data dependent algorithm
	for selecting a test statistic $t\in\mathcal{T}$. Let $\mathbf{X}$ be a random dataset over $\mathcal{X}^n$. Suppose that $\sigma\in [0, 1]$ is the significance level chosen to control the false discovery probability for the test statistic $t$. Denote
	with $E$ the event that $\mathcal{A}$ selects a statistic such that the null
	hypothesis is true but its p-value is at most $\sigma$. Then,
	\begin{equation}\mathbb{P}(E) \leq \exp(\Le(\mathbf{X} \to \mathcal{A}(\mathbf{X}))) \cdot \sigma\end{equation}
\end{mytheorem}
This result suggests a very simple approach: if the analyst wishes to achieve a bound of $\delta$ on the probability of making a false discovery in adaptive settings, the significance level $\sigma$ to be used should be no higher than $\delta/\exp(\Le(\mathbf{X}\to \A(\mathbf{X}))$. Suppose, as typically done in the field, that we wish to achieve an upper-bound of $\delta=0.05$. Whenever we intend to choose a statistic from a set $\mathcal{T}$, one straight-forward approach would be to use as a significance value the quantity $0.05/|\mathcal{T}|$. With $|\mathcal{T}|=10$, this would imply choosing a significance level $\sigma=0.005$. The application of Theorem \ref{hypTestML} is thus almost immediate.
Once again, if $\mathcal{A}$ is independent from $\mathbf{X}$, we recover the classical bound of $\sigma$. Also, Theorem \ref{hypTestML} is similar to a result obtained in \citep{maxInfo} and that uses the notion of $\beta-$approximate max-information instead. 
\subsection{Maximal Leakage and Mutual Information}
\label{sec:leakageAndMutI}
One interesting result in the field, that connects the generalization error with Mutual Information, under the same assumptions of Theorem~\ref{generrML}, is the following (Theorem 8 of \citep{learningMI}): 
\begin{equation}\label{generrMI}
\mathbb{P}(E) \leq \frac{I(S;\mathcal{A}(S))+\log 2}{2n\eta^2-\log 2}.
\end{equation}
\iffalse here recalled for convenience:\begin{mylemma}\label{generrMI}\citep{learningMI}
	Let $\mathcal{A}:\mathcal{X}^n\to \mathcal{H}$ be a learning algorithm defined on the domain $\mathcal{X}$ and that, given a sequence $S$ of $n$ points sampled in an iid fashion from $\mathcal{X}^n$ according to a distribution $\mathcal{P}$, i.e. $S\sim \mathcal{P}^n$, returns a hypothesis $h\in \mathcal{H}$. \\
	Let $E=\{(S,h):|L_{\mathcal{P}}(h)-L_S(h)|>\epsilon \}$.\\  
	Let $\mu$ be a distribution on the pairs $(S,h)$ where $S\sim \mathcal{P}^n$ and $h=\mathcal{A}(S)$:
	$$ \mu(E) \leq \frac{I(S;\mathcal{A}(S))+1}{2n\epsilon^2-1}.$$
\end{mylemma}
\fi
Let us compare this result with Theorem~\ref{generrML} in terms of sample complexity.
\begin{mydef}
	Fix $\eta, \delta \in (0,1)$. Let $\mathcal{H}$ be an hypothesis class. The sample complexity of $\mathcal{H}$ with respect to $(\eta,\delta)$, denoted by $m_\mathcal{H}(\epsilon,\delta)$,
	is defined as the smallest $n \in \mathbb{N}$ for which there exists a learning algorithm $\mathcal{A}$ such that, for every distribution $\mathcal{P}$ over the domain $\mathcal{X}$ yields
	\begin{equation} \mathbb{P}(\text{gen-err}_\mathcal{P}(\mathcal{A},S)>\eta)\leq \delta. \end{equation}
	If there is no such $n$ then $m_\mathcal{H}(\eta,\delta)=\infty$.
\end{mydef}
\noindent
From Theorem~\ref{generrML}, it follows that using a sample size of $n\geq \left(\frac{ \mathcal{L}(S \to \mathcal{A}(S))+\ln(1/\delta)}{\eta^2}\right)$ yields a learner for $\mathcal{H}$ with accuracy $\eta$ and confidence $\delta$ and this, in turn, implies that $m_\mathcal{H}(\eta,\delta)= O\left(\frac{ \mathcal{L}(S \to \mathcal{A}(S))+\ln(1/\delta)}{\eta^2}\right)$. Using the same reasoning with inequality~\eqref{generrMI}, we get  $m_\mathcal{H}(\eta,\delta)= O\left(\frac{I(\A(S);S)}{\eta^2}\cdot \frac{1}{\delta}\right).$ The reduction in the sample complexity, in this regime, is exponential in $\delta$. Moreover, as shown in \citep{learningMI}, if we consider the case where $\mathcal{X}=[d]$ and $\mathcal{H}=\{0,1\}^{\mathcal{X}}$, we have that the VC-dimension of $\mathcal{H}$ is $d$ and, being  $\Le(S\to\mathcal{A}(S))\leq \log(|\mathcal{H}|)\leq d$, our bound recovers  exactly the VC-dimension bound~\citep{learningBook}, which is always sharp.
\iffalse
The sample complexity that Theorem \ref{generrML} provides can be found solving the inequality $2\cdot\exp( \mathcal{L}(S \to \mathcal{A}(S)) -2n\eta^2)\leq \delta$ by $n$ and results in the following lower bound $n=\Omega\left(\frac{ \mathcal{L}(S \to \mathcal{A}(S))+\ln(1/\delta)}{\eta^2}\right)$. Using  inequality~\eqref{generrMI} instead, we get  $n=\Omega\bigg(\frac{I(\A(S);S)}{\eta^2}\cdot \frac{1}{\delta}\bigg).$ 
The reduction in the sample complexity, in this regime, is exponential in $\delta$. Moreover, if the VC-dimension of $\mathcal{H}$ is equal to $\log|\mathcal{H}|$, then our bound recovers the VC-dimension bound~\citep{learningBook} since  $\Le(S\to\mathcal{A}(S))\leq \log(|\mathcal{H}|)$. \fi

%\begin{equation} \Le(S\to\mathcal{A}(S))\leq \log(|\mathcal{H}|) \approx k,\end{equation}
%$\text{VCdim}(\mathcal{H})\leq k$ and the two bounds coincide.\\
%More generally, whenever $|\mathcal{H}|<\infty$ the Mutual Information approach and the Maximal Leakage one, allow us to make non-trivial statements in terms of generalization error, since both measures are bounded by $\log(|\mathcal{H}|)$. When this happens, the reasoning briefly exposed above applies and we thus retrieve the exponential improvement in the sample complexity.%It is also clear that requiring a bounded Maximal Leakage is a stronger constraint than requiring a bounded Mutual Information. 
\subsection{Maximal Leakage and Max Information}
\label{sec:leakageAndMI}
%In this section, we will explore the relationship between maximal leakage and ($\beta$-)approximate max-information. 
One of the main reasons that brought to the definition of approximate max-information is related to the generalization guarantees it provides, now recalled for convenience.
\begin{mylemma}[Generalization via Max-Information, Thm. 13 of \citep{genAdap}]\label{GenMaxInfo}
	Let $\mathbf{X}$ be a random dataset in $\X^n$ and let $\A:\X^n\to \Y$ be such that for some $\beta\geq 0$, $I_\infty^\beta(\mathbf{X},\A(\mathbf{X}))=k$ then, for any event $\mathcal{O}\subseteq X^n\times \Y$:
	\begin{align}\Prob_{(\mathbf{X},\A(\mathbf{X}))}(\{(\mathbf{X},\A(\mathbf{X}))\in\mathcal{O}\})\leq e^k\cdot\Prob_{\mathbf{X} \times \A(\mathbf{X})}(\{(\mathbf{X},\A(\mathbf{X}))\in\mathcal{O}\})+ \beta \end{align}
\end{mylemma}
The result looks quite similar to Theorem \ref{generrML}, but the two measures, max-information and maximal leakage, although related, can be quite different. In this section we will analyze the connections and differences between the two measures underlining the corresponding implications.
Thanks to the constraint on the distributions that a bound on max-information imposes, we can formalize the following connection.
\begin{mytheorem}
	Let $\mathcal{A}:\mathcal{X}^n\to\mathcal{Y}$ be a randomized algorithm such that $I_\infty(\mathbf{X};\mathcal{A}(X))\leq k$. Then, $\Le(\mathbf{X}\to\mathcal{A}(X))\leq k.$
\end{mytheorem}
\begin{proof}$Y=\A(\mathbf{X})$. Having a bound of $k$ on the Max-Information of $\mathcal{A}$ means that for all $x^n \in\mathcal{X}^n,$ and $y \in \mathcal{Y}, \mathbb{P}(Y=y|X^n=x^n)\leq e^k \cdot \mathbb{P}(Y=y)$; and this implies that $\Le(\mathbf{X}\to Y)\leq k.$
\end{proof} \noindent
With respect to $\beta$-approximate max-information instead, we can state the following.
\begin{mytheorem}\label{bMIandML}
	Let $\mathcal{A}:\mathcal{X}^n\to\mathcal{Y}$ be a randomized algorithm. Let $\mathbf{X}$ be a random variable distributed over $\mathcal{X}^n$ and let $Y=\mathcal{A}(X)$. For any $\beta \in (0,1)$ \begin{equation} I_{\infty}^\beta(\mathbf{X};\A(\mathbf{X})) \leq \Le(\mathbf{X}\to \A(\mathbf{X})) + \log\bigg({\frac{1}{\beta}}\bigg).\end{equation}
\end{mytheorem}
Before showing this theorem we need the following, intermediate result. We denote with $\mathcal{P}_X$ the distribution associated with $X$.
\begin{mylemma}[Lemma 18 of \citep{genAdap} ]\label{maxDivProb}
	Let $X,Y$ be two random variables over the same domain $\mathcal{X}$. If $\mathbb{P}_{x\sim \mathcal{P}_X}\bigg(\frac{\mathbb{P}(X=x)}{\mathbb{P}(Y=x)}\geq e^k \bigg)\leq \beta$ then $D_\infty^\beta(\mathcal{P}_X||\mathcal{P}_Y)\leq k$.
\end{mylemma}
\begin{proof}
	Fix any $\beta>0$. Denote with $Y=\A(\mathbf{X})$, we want to show that $D_\infty^\beta (\mathcal{P}_{(\mathbf{X},\A(\mathbf{X}))}||\mathcal{P}_{\mathbf{X}}\mathcal{P}_{\A(\mathbf{X})})\leq \Le(\mathbf{X}\to Y) + \log{\big({\frac{1}{\beta}}\big)}$ using Lemma \ref{maxDivProb}. 
	Notice that $\Le(\mathbf{X}\to Y)=\log\E_Y\bigg[\sup_{x\in\mathcal{X}^n}\frac{\mathbb{P}(\mathbf{X}=x, Y=y)}{\mathbb{P}(\mathbf{X}=x)\mathbb{P}(Y=y)}\bigg]$.\\
	Using Markov Inequality we can proceed in the following way:
	\begin{align}
	\mathbb{P}_{(x,y)\sim \mathcal{P}_{XY}}\bigg(\frac{\mathbb{P}(\mathbf{X}=x, Y=y)}{\mathbb{P}(\mathbf{X}=x)\mathbb{P}(Y=y)}\geq \frac{e^{\Le(\mathbf{X}\to Y)}}{\beta} \bigg) &\leq \frac{\mathbb{E}_{XY}\bigg[\frac{\mathbb{P}(\mathbf{X}=x, Y=y)}{\mathbb{P}(\mathbf{X}=x)\mathbb{P}(Y=y)}\bigg]\cdot \beta}{e^{\Le(\mathbf{X}\to Y)}} \\ &\leq \frac{\E_Y\bigg[\sup_{x\in\mathcal{X}^n}\frac{\mathbb{P}(\mathbf{X}=x, Y=y)}{\mathbb{P}(\mathbf{X}=x)\mathbb{P}(Y=y)}\bigg]\cdot \beta}{e^{\Le(\mathbf{X}\to Y)}} \\ &= \beta.
	\end{align} 
	Hence, $I_{\infty}^\beta(\mathbf{X};\A(\mathbf{X})) = D_\infty^\beta (\mathcal{P}_{(\mathbf{X},\A(\mathbf{X}))}||\mathcal{P}_{\mathbf{X}}\mathcal{P}_{\A(\mathbf{X})})\leq \log\big(\frac{e^{\Le(\mathbf{X}\to Y)}}{\beta} \big) = \Le(\mathbf{X}\to Y) + \log\big(\frac{1}{\beta}\big).$
\end{proof}
Here, the relationship between the measures is inverted, but the role played by $\beta$ can lead to undesirable behaviours of $\beta$-approx MI. 
The following example, indeed, shows how $\beta$-approx MI can be unbounded while, in the discrete case, the maximal leakage between two random variables is always bounded by the logarithm of the smallest cardinality.
\begin{examp}
Let us fix a $\beta \in (0,1).$
Suppose $X\sim\text{Ber}(2\beta)$. We have that $\ml{X}{X}=\log |\text{supp}(X)|=\log2$. For the $\beta-$approximate max-information we have:
$I_\infty^\beta(X;X) \geq \log ((2\beta-\beta)/\beta^2) = \log(1/\beta)$. It can thus be arbitrarily large.
\end{examp}
%Since it is rather unclear how to compute the $\beta-$approximate max-information when the prior distribution is unknown (as typically is the case), the approach undertaken in~\citep{genAdap,maxInfo} is to upper-bound the measure through differential privacy or the description length of the algorithm. These bounds allow for an explicit characterization of the generalization guarantees that max-information provides.
Another interesting characteristic of max-information is that, differently from differential privacy, it can be bounded even if we have deterministic algorithms.
It is easy to see that whenever there is a deterministic mapping and $\epsilon$-Differential Privacy is enforced on it, a lower bound on $\epsilon$ of $+\infty$ is retrieved. Trying to relax it to $(\epsilon,\delta)-$Differential Privacy does not help either, as one would need $\delta\geq 1$ rendering it practically useless. Max-information, instead, can be finite, and this observation is implied by the connection with what in literature is known as \enquote{description length} of an algorithm, and synthesized in the following result \citep{genAdap}:
Let $\mathcal{A}:\mathcal{X}^n\to\mathcal{Y}$ be a randomized algorithm, for every $\beta>0$, \begin{equation} I_{\infty}^{\beta}(\mathcal{A},n)\leq \log{\bigg(\frac{|\mathcal{Y}|}{\beta}\bigg)}.\label{descrMI}\end{equation}
In terms of generalization guarantees, according to Lemma \ref{GenMaxInfo}, this translates to:
\begin{align}\label{bmaxinfo}
\Prob_{(\mathbf{X},\A(\mathbf{X}))}(\{(\mathbf{X},\A(\mathbf{X}))\in\mathcal{O}\})\leq \frac{|\Y|}{\beta}\cdot\Prob_{\mathbf{X} \times \A(\mathbf{X})}(\{(\mathbf{X},\A(\mathbf{X}))\in\mathcal{O}\})+ \beta
\end{align}
In contrast, the bound on maximal leakage that bounded description length provides is the following:
\begin{mylemma}\label{mlAndBDL}
	Let $\mathcal{A}:\mathcal{X}^n\to\mathcal{Y}$ be a randomized algorithm, then $\Le(\mathbf{X}\to\mathcal{A}(\mathbf{X}))\leq \log(|\mathcal{Y}|)$.
\end{mylemma} 
Lemma \ref{mlAndBDL} along with Theorem \ref{adaptML}, immediately translates in the following generalization guarantees:
\begin{align}
\Prob_{(\mathbf{X},\A(\mathbf{X}))}(\{(\mathbf{X},\A(\mathbf{X}))\in\mathcal{O}\})\leq |\Y|\cdot\max_{Y=\A(\mathbf{X})}\Prob_{\mathbf{X} }(\{(\mathbf{X}\in\mathcal{O}_Y\})
\end{align}
The bound is different from the one in \eqref{bmaxinfo} as on the right-hand side we have \\$\max_{Y=\A(\mathbf{X})}\Prob_{\mathbf{X} }(\{(\mathbf{X}\in\mathcal{O}_Y\})$ while in  \eqref{bmaxinfo} appears $\E_{Y=\A(\mathbf{X})}\Prob_{\mathbf{X} }(\{(\mathbf{X}\in\mathcal{O}_Y\})$. For practical purposes though, the quantity is typically bounded for every $Y$, as for instance when using McDiarmid's inequality, like in Theorem \ref{generrML}. This would imply a bound on both the max and the expectation, and when such technique is used then our bound is clearly tighter as $0<\beta\leq 1$.
It is also worth noticing that \eqref{descrMI} can be seen as a consequence of Theorem \ref{bMIandML} and Lemma \ref{mlAndBDL}.
 %we are able to provide a better bound in terms of generalization capabilities, through Maximal Leakage, given a bound on the Description Length of $\mathcal{A}$.
 
 Once again, due to the presence of $\beta$, maximal leakage can provide tighter bounds, as we can see from the following example.
\begin{examp}
Let $\A : \X \to \Y$ be an algorithm. Denoting with $|\Y|$ the size of the range of $\A$, we can bound Approximate Max-Information by $\log(|Y|/\beta)$ while maximal leakage is always bounded by $\log(|\Y|)$. Since $\beta$ is typically very small in the key applications, and the corresponding multiplicative factors in the bounds are $(|Y|/\beta)$ and $|\Y|$, the difference between the two bounds can be substantial. 
\end{examp}
It is important to also notice that the difference between the two measures is not uniquely restricted to deterministic mechanisms. The following is an example of a randomized mapping where maximal leakage is smaller than $\beta$-approximate-max-information, for small $\beta$.
\begin{examp}
Consider $X\sim\text{Ber}(1/2)$ and $Y=BEC_\alpha(X)$ (the output of a Binary Erasure Channel, with erasure probability $\alpha$, when X is transmitted). More formally, we have that $\Y=\{0,1,e \}$ and the following randomized mapping: $\Prob(Y=e|X=x) = \alpha$ and $\Prob(Y=x|X=x)=1-\alpha$.
In this case, the Maximal Leakage is $\ml{X}{Y}=\log(2-\alpha)$ \citep{leakageLong}; while, for $\beta$-Approximate Max-Information one finds (after a series of computations) that:  
                         $I_\infty^\beta(X;Y)=\log(2\cdot\max\{(1-\alpha-\beta)/(1-\alpha),  (1-\beta)/(1+\alpha)\});$
It is easy to see how for a fixed $\alpha$ and for $\beta$ going to $0$, Approximate Max-Information approaches $\log2$ while Maximal Leakage is strictly smaller.
\end{examp}
 Since we are mostly interested in the regimes where $\beta$ is small, these examples show how using Maximal Leakage provides tighter bounds and, in general, a wider applicability of our results as, in discrete settings, it is always bounded.

\newpage

\appendix
\section{Definitions and Tools}
\label{app:definitions}

\begin{mydef}
	Let $X,Y$ be two discrete random variables defined over the same alphabet $\mathcal{X}$, the max-divergence of $X$ from $Y$ is defined as:
	\begin{equation}D_{\infty}(X||Y)= \log{\max_{x\in\mathcal{X}}\frac{\Prob(\{X=x\})}{\Prob(\{Y=x\})}},\end{equation}
	while, the $\delta-$approximate max-divergence is defined as:
	\begin{equation}D_{\infty}^\delta(X||Y)= \log{\max_{\mathcal{O}\subseteq \mathcal{X}, \Prob(\{X\in\mathcal{O}\})>\delta}\frac{\Prob(\{X\in \mathcal{O}\})-\delta}{\Prob(\{Y\in\mathcal{O}\})}}.\end{equation}
\end{mydef}

\begin{mydef}[Differential Privacy \citep{genAdap}]
	Let $\epsilon,n>0$. A randomized algorithm $\mathcal{A}:\mathcal{X}^n \to \mathcal{Y}$ is said to be $\epsilon$-differentially private if, for all the pairs of datasets that differ in a single component $\mathbf{X},\mathbf{X}' \in \mathcal{X}^n$ we have that $D_{\infty}(\mathcal{A}(\mathbf{X)}||\mathcal{A}(\mathbf{X}'))\leq \epsilon.$
\end{mydef}
An important probability tool, often used in the field is McDiarmid's inequality, a concentration of measure for functions of a certain sensitivity. We will say that a function $f:\mathcal{X}^n\to\mathbb{R}$ has sensitivity $c$ if 
\begin{equation} |f(x_1,\ldots,x_i,\ldots,x_n)-f(x_1,\ldots,x'_i,\ldots,x_n)|\leq c\, \quad \forall x_1,\ldots,x_n,x'_i\in\mathcal{X}^{n+1}.\end{equation}
\begin{mylemma}[McDiarmid's inequality]\label{mcdiarmids}
	Let $X_1,\ldots,X_n$ be independent random variables taking values in the common alphabet $\mathcal{X}$. \\
	Let $f:\mathcal{X}^n\to \mathbb{R}$ be a function of sensitivity $c>0$.
	For every $\epsilon>0,$
	\begin{equation}\mathbb{P}(\{f(X_1,\ldots,X_n)-\E[f(X_1,\ldots,X_n)]\geq \epsilon \})\leq \exp\bigg(\frac{-2\epsilon^2}{n\cdot c^2}\bigg).\end{equation}
\end{mylemma}
%The constraint that this definition introduces on the distribution represents some form of stability that allows us to obtain generalization guarantees.
%Another, different, approach to retrieve generalization guarantees is the following. \\ Let $\A :\X^n\to \Y, \B:\X^n\times \Y\to \Y' $ be two algorithms. Let us denote with $R(y)\subseteq \X^n, y\in\Y$ the datasets for which the outcome of $\B(\cdot,y)$ behaves \enquote{badly}. If, for every $y\in \Y$, the probability of the set $R(y)$ is small, then a simple union bound allows us to control the probability, even in the adaptive settings.
%\begin{mytheorem} [Thm. 9 of \citep{genAdap}]
%	Let $\A:\X^n \to \Y$ be an algorithm and $\mathbf{X}$ be a random dataset over $\X^n$. Let $R:\Y\to 2^{\X^n}$. Suppose that $\forall y\in \Y, \Prob(\{\mathbf{X}\in R(y) \})\leq \beta$, we have that $\Prob(\{\mathbf{X}\in R(Y) \})\leq |\Y|\cdot \beta$, with $Y=\A(\mathbf{X}).$
%\end{mytheorem}
%This result can be easily extended to the general case of more algorithms.
%A measure that seems to bring an unifying view of the two main techniques we have just mentioned is the Max-Information.
%We are considering two algorithms $\A,\B$ and their adaptive composition. The main idea here is that, if $\A$ doesn't reveal too much information on the dataset $\mathbf{X}$, then we will be able to preserve the generalization capabilities of $\B$.
\begin{mydef}[Def. 10 of \citep{genAdap}]
	Let $X,Y$ be two random variables jointly distributed according to $\mathcal{P}_{XY}$. Then, the max-information between $X$ and $Y$, is defined as follows:
	\begin{equation}I_\infty(X;Y) = D_\infty(\mathcal{P}_{XY}||\mathcal{P}_{X}\mathcal{P}_{Y}),\end{equation}
	while, the $\beta$-approximate max-information is defined as:
	\begin{equation}I_\infty^\beta(X;Y) = D^\beta_\infty(\mathcal{P}_{XY}||\mathcal{P}_{X}\mathcal{P}_{Y}).\end{equation}
\end{mydef}

%Considering the joint distribution $(\mathbf{X},\A(\mathbf{X}))$ the idea would be that, bounding the leakage of information of $\mathbf{X}$ from $\A$  we will be able to preserve the generalization capabilities of $\B$, since 

\section{Proof of Theorem \ref{adaptML}}
\label{app:proof}
\begin{proof}
	Fix $y \in \mathcal{Y}$, we have the following two distributions on $\mathcal{X}$: $\mathcal{P}_{X|Y=y}, \mathcal{P}_X$. Denote with $\beta=\max_{y\in\Y} \mathcal{P}_X(E_y).$\\
	By definition of R\'enyi Divergence of order $\infty$ we have \citep{RenyiKLDiv}, 
	\begin{align}
	D_{\infty}(\mathcal{P}_{X|Y=y}|| \mathcal{P}_X) = \log \sup_{A\subseteq \X} \frac{\mathcal{P}_{X|Y=y}(A)}{\mathcal{P}_X(A)} = \log \sup_{x: P_{X|Y=y}(x)>0} \frac{\mathcal{P}_{X|Y=y}(x)}{\mathcal{P}_X(x)}
	\end{align}
	We can say that, for every $y\in\Y$:
	\begin{align}\mathcal{P}_{X|Y=y}(E_y) &= \frac{\mathcal{P}_{X|Y=y}(E_y)}{\mathcal{P}_X(E_y)}\cdot \mathcal{P}_X(E_y) \\
	&\leq \exp(D_{\infty}(\mathcal{P}_{X|Y=y}|| \mathcal{P}_X)) \cdot \mathcal{P}_X(E_y) \\ &\leq\exp(D_{\infty}(\mathcal{P}_{X|Y=y}|| \mathcal{P}_X)) \cdot \beta,
	\end{align}
	Taking expectation on both sides we get:	\begin{align}\E_Y(\mathcal{P}_{X|Y=y}(E_y)) &\leq \beta \cdot \E_Y(\exp(D_{\infty}(\mathcal{P}_{X|Y=y}|| \mathcal{P}_X))) \iff \\
	\mathbb{P}(E) &\leq \max_{y\in\Y} \mathcal{P}_X(E_y) \cdot \exp(\Le(X\to Y)) \label{eq: lastStepProof}.
	\end{align}
	Where \eqref{eq: lastStepProof}, in discrete settings, follows from:
	\begin{align}
	\E_Y\big(\exp(D_{\infty}(\mathcal{P}_{X|Y=y}|| \mathcal{P}_X))) &=  \E_Y\bigg(\sup_{A\subseteq \X} \frac{\mathcal{P}_{X|Y=y}(A)}{\mathcal{P}_X(A)}\bigg) \\
	&= \E_Y\bigg(\max_{x: P_{X|Y=y}(x)>0} \frac{\mathcal{P}_{X|Y=y}(x)}{\mathcal{P}_X(x)}\bigg)\\
	&= \sum_{y\in\Y:P_Y(y)>0}\max_{x: P_{X|Y=y}(x)>0}\frac{\mathcal{P}_{X,Y}(x,y)}{\mathcal{P}_X(x)} \\
	&= \sum_{y\in\Y:P_Y(y)>0}\max_{x\in\X : P_X(x)>0}\mathcal{P}_{Y|X=x}(y) \label{tech}\\
	&= \sum_{y\in\Y}\max_{x\in\X : P_X(x)>0}\mathcal{P}_{Y|X=x}(y) \\
	&= \exp(\Le(X\to Y))
	\end{align}
	And \eqref{tech} follows from: 
	\begin{align*}
	   P_Y(y)>0 \text{ and } P_{X|Y}(x|y)>0 &\implies P_X(x)>0 \\
	   P_Y(y)>0 \text{ and } P_{X|Y}(x|y)=0 &\implies P_{Y|X}(y|x)=0 
	\end{align*}
\end{proof}
\section{Maximal Leakage properties}
\label{app:leakageProps}
The following is a proof of Theorem \ref{thComposition2}.
\begin{proof}
	If we consider the second constraint in our assumption and denoting with $Z_y = \mathcal{B}(X,y)$, we get:
	\begin{align}\forall y \in \mathcal{Y} \Le(X\to Z_y) \leq k_2 &\equiv \forall y \in \mathcal{Y} \sum_{z_y\in supp\{Z_y\}} \max_{x\in supp\{X\}}  \mathbb{P}(z_y|x) \leq \exp(k_2) \\ &\equiv \forall y \in \mathcal{Y} \sum_{z_y\in supp\{Z|Y=y\}} \max_{x\in supp\{X\}}  \mathbb{P}(z|x,y) \leq \exp(k_2) .\end{align}
	The last step holds, since every $y$ generates a family of conditional distributions $\mathbb{P}(z_y|x)$ through $\mathcal{B}$ and this probability \iffalse, that accounts only for the randomness of $\mathcal{B}$,\fi is equivalent to $\mathbb{P}(z|x,y)$,  with $z=\mathcal{B}(x,y)$.\\
	Using this observation in the conditional leakage of (\ref{ineqLeak}):
	\begin{align}
	\Le(X\to Z|Y) &= \log \max_{y\in supp\{Y\}} \sum_{z \in supp\{Z|Y=y\}} \max_{x \in supp\{X|Y=y\}} \mathbb{P}(z|x,y) 
	\\ &\leq \log \max_{y\in supp\{Y\}} \sum_{z \in supp\{Z|Y=y\}} \max_{x \in supp\{X\}} \mathbb{P}(z|x,y) 
	\\ &\leq \log \max_{y\in supp\{Y\}} \exp(k_2) \\ &= k_2,
	\end{align}
	leading us to the desired bound.
\end{proof}
Here is a proof of Lemma \ref{compositionN}.
\begin{proof}
	\begin{align}
	\Le(X\to (A_1, \ldots, A_n)) &= \Le(X \to \mathbf{A}^n) \\ &= \Le(X \to (\textbf{A}^{n-1},A_n))\\
	&\leq \Le(X\to \textbf{A}^{n-1}) + \Le(X \to A_n |\textbf{A}^{n-1} ) \\ &\leq \Le(X\to \textbf{A}^{n-2}) + \Le(X\to A_{n-1}|\textbf{A}^{n-2}) + \Le(X \to A_n |\textbf{A}^{n-1} ) \\ &\leq \ldots \\ &\leq \Le(X\to A_1) + \Le(X\to A_2|A_1)+ \ldots + \Le(X\to A_n| (A_1,\ldots,A_{n-1})).
	\end{align}
\end{proof}
To conclude a proof of Lemma \ref{MLandDP}.
\begin{proof}
	Let $Y=\mathcal{A}(X)$. Fix some $\hat{\mathbf{x}}\in\mathcal{X}^n$, $\forall\,\mathbf{x}\in\mathcal{X}^n$ we have that $\mathbf{x}$ and $\hat{\mathbf{x}}$ differ in at most $n$ positions and, iteratively applying the definition of Differential Privacy, we have that\\ $\mathbb{P}(Y=y|X=\mathbf{x})\leq e^{\epsilon\cdot n} \mathbb{P}(Y=y|X=\hat{\mathbf{x}}).$\\
	\begin{align}\Le(X\to Y) &= \log \sum_{y\in\mathcal{Y}} \max_{\mathbf{x}\in\mathcal{X}^n} \mathbb{P}(Y=y|X=\mathbf{x})\\ &\leq \log\sum_{y\in\mathcal{Y}} e^{\epsilon\cdot n}\mathbb{P}(Y=y|X=\hat{\mathbf{x}}) \\ &= n\cdot\epsilon \end{align}
\end{proof}
\vskip 0.2in
\bibliography{sample}
\end{document}